\documentclass[letterpaper, 10 pt, conference]{ieeeconf}  

\IEEEoverridecommandlockouts                              
\usepackage{graphics} 
\usepackage{epsfig} 
\usepackage{mathptmx} 
\usepackage{times} 
\usepackage{amsmath} 
\usepackage{amsfonts}
\usepackage{amssymb}  
\usepackage{color}
\usepackage{verbatim}
\usepackage{comment}
\usepackage{fancyhdr}
\usepackage{xcolor}
\usepackage[ruled,vlined]{algorithm2e}
\usepackage{algorithmic}
\usepackage{optidef}
\usepackage{bbold}
\usepackage{cite}
\usepackage{float}
\usepackage{todonotes}

\newtheorem{theorem}{Theorem}[section]

\newtheorem{lemma}[theorem]{Lemma}
\newtheorem{assumption}[theorem]{Assumption}
\usepackage{hyperref}
\hypersetup{
    colorlinks=true,
    linkcolor=blue,
    filecolor=magenta,      
    urlcolor=cyan,
    pdftitle={Overleaf Example},
    pdfpagemode=FullScreen,
    }

\usepackage[compact]{titlesec}
\titlespacing{\section}{2pt}{*+1}{*+1}
\titlespacing{\subsection}{2pt}{*+1}{*+1}
\titlespacing{\subsubsection}{2pt}{*+1}{*+1}

\title{RETRO: Reactive Trajectory Optimization for Real-Time \\ Robot Motion Planning in Dynamic Environments}

\author{Apan Dastider$^{1}$, Hao Fang$^{1}$ and Mingjie Lin$^{1}$
\thanks{$^{1}$Department of Electrical and Computer Engineering, University of Central Florida, Orlando, FL, 32816, USA (E-mail: milin at ucf.edu)}}

\begin{document}

\setlength{\textfloatsep}{10pt}

\maketitle
\thispagestyle{empty}
\cfoot{\thepage}
\renewcommand{\headrulewidth}{0pt}
\pagestyle{empty}
\cfoot{\thepage}
\setlength{\textfloatsep}{3pt}
\setlength{\floatsep}{3pt}

\begin{abstract}
Reactive trajectory optimization for robotics presents formidable challenges, demanding the rapid generation of purposeful robot motion in complex and swiftly changing dynamic environments. While much existing research predominantly addresses robotic motion planning with predefined objectives, emerging problems in robotic trajectory optimization frequently involve dynamically evolving objectives and stochastic motion dynamics. However, effectively addressing such reactive trajectory optimization challenges for robot manipulators proves difficult due to inefficient, high-dimensional trajectory representations and a lack of consideration for time optimization.

In response, we introduce a novel trajectory optimization framework called RETRO. RETRO employs adaptive optimization techniques that span both spatial and temporal dimensions. As a result, it achieves a remarkable computing complexity of $O(T^{2.4}) + O(Tn^{2})$, a significant improvement over the traditional application of DDP, which leads to a complexity of $O(n^{4})$ when reasonable time step sizes are used.
To evaluate RETRO's performance in terms of error, we conducted a comprehensive analysis of its regret bounds, comparing it to an Oracle value function obtained through an Oracle trajectory optimization algorithm. Our analytical findings demonstrate that RETRO's total regret can be upper-bounded by a function of the chosen time step size.
Moreover, our approach delivers smoothly optimized robot trajectories within the joint space, offering flexibility and adaptability for various tasks. It can seamlessly integrate task-specific requirements such as collision avoidance while maintaining real-time control rates. We validate the effectiveness of our framework through extensive simulations and real-world robot experiments in closed-loop manipulation scenarios.

For further details and supplementary materials, please visit:
\url{https://sites.google.com/view/retro-optimal-control/home}


\end{abstract}

\section{INTRODUCTION}
Reactive Trajectory Optimization (RETRO) is a distinct paradigm in robotic control, characterized by dynamic changes in system dynamics and optimization/control objectives over time. In this approach, the full set of system state data is not available simultaneously.
The progression of the control process is represented as a sequence of system control policies. The primary aim is to continuously and effectively refine the currently computed trajectory throughout this sequence while intelligently incorporating knowledge from past experiences.
To be more specific, the overarching objective of RETRO is to minimize a predefined cost function over an entire epoch of robotic control. This remains true even when making multiple adjustments or refinements to the presently computed control policy.
In essence, RETRO can be conceptualized as an online learning process that necessitates the fusion of knowledge from sequentially presented data streams over time.
Practically, RETRO also seeks to optimize memory usage, computational resources, and execution speed during the robotic trajectory optimization process, making it a comprehensive and efficient approach to robotic control in dynamic and evolving scenarios.


\begin{figure}[htbp]                                       
    \centering                                             
    \includegraphics[width=\linewidth]{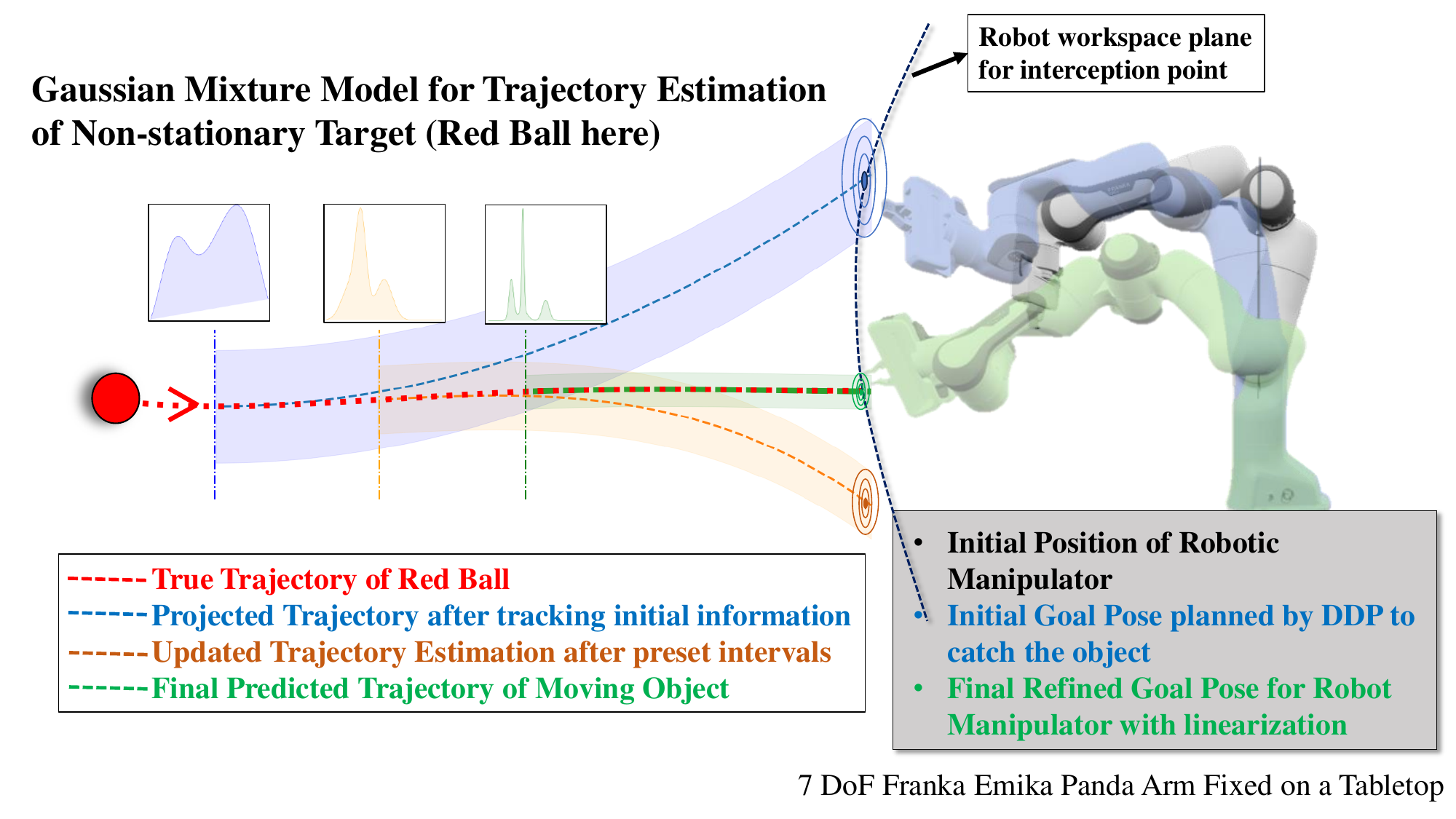}  
    \caption{
        Overall Experimental Platform for RETRO. A Robotic Research-grade 7 DoF Robotic Manipulator has been used and a ball is thrown towards the manipulator to intercept it. As more information for ball tracking is available, the GMM model becomes more accurate to predict the future trajectory. Variance gets smaller as more data is available for trajectory prediction. With updated trajectory information, the action sequence also gets refined through computation-efficient action adjustment method.   
        }%
    \label{fig:intro}                                   
\end{figure} 

\begin{figure}[htbp]                                       
    \centering                                             
    \includegraphics[width=\linewidth]{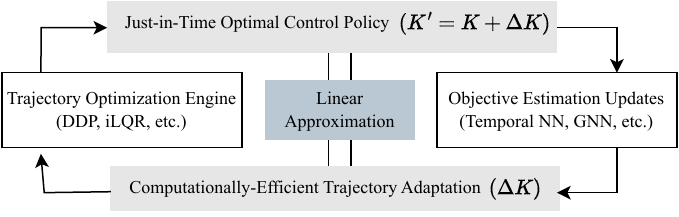}  
    \caption{
        Conceptual Diagram of RETRO methodology.
        }%
    \label{fig:method}                                   
\end{figure}  

\begin{figure*}[htbp]
    \includegraphics[width=1\textwidth]{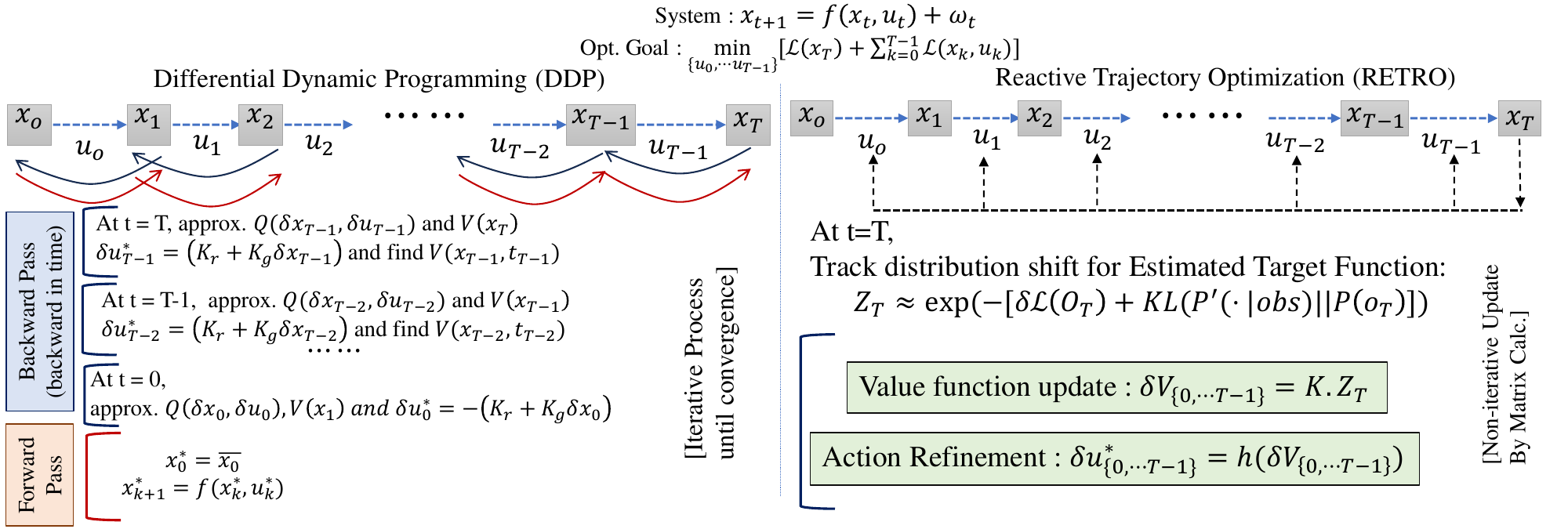}
    \caption{A Comparative Analysis Between DDP-based Action Refinement and RETRO suggested Optimal Action Sequence Planning. The left panel shows that a multiple-run DDP is required for a non-stationary target. In the right panel, the proposed RETRO upgrades the actions based on the initially planned trajectory and value function update calculated through linearized calculation.}
    \label{fig:comparison}
\end{figure*}

Practical applications of Reactive Trajectory Optimization (RETRO) are diverse and encompass a wide range of scenarios. For instance, consider a mobile robot tasked with reaching a goal pose that swiftly moves along an unpredictable yet discernible path. In another context, a navigating drone endeavors to safely explore a cluttered environment, equipped with sensors featuring severely limited detection ranges. Consequently, it can only compute its control policy incrementally. Taking a more imaginative perspective, envision a highly intelligent missile interceptor seeking to intercept a highly maneuverable hypersonic missile warhead.
In this research paper, our primary focus is on the development of an efficient real-time control algorithm. This algorithm empowers a 7-degree-of-freedom (7-DOF) Franka robotic arm to consistently and successfully catch a flying subject that approaches it. Our approach distinguishes itself from several analogous problem settings~\cite{x1, x3, rals, DVS, apan, objectsFlight, NEURIPS2022_4b70484e} by a crucial assumption: the trajectory of the flying subject, while predictable, can only be estimated incrementally and with precision challenges. This limitation arises from the early stages of robotic trajectory optimization, where the scarcity of sensory data renders the prediction of the flying subject's final impact point highly imprecise and inherently probabilistic in nature. Nonetheless, the trajectory optimization process must remain continuous throughout the entire operational epoch.

\section{Related Work}


In robotics, Differential Dynamic Programming (DDP) has been extensively studied 
to compute motion planning and control,
especially solving nonlinear and time-varying optimization problems in real-time applications. 
Specifically, multiple works~\cite{x1,x2,x3,x4} studied how to effectively apply DDP  
to control systems subject to nonlinear constraints, such as safety
contraints~\cite{x1} and feasibility constraints~\cite{x2,x3}, parametric uncertainty~\cite{RH_DDP}, 
as well as how to account for uncertainties and
stochasticity in the system dynamics~\cite{x4}.
More recently, multiple studies attempted to improve 
the efficiency and applicability of DDP-based control methods.
For example, study~\cite{x5} explored the integration of data-driven approaches and Gaussian
processes into DDP in order to enhance
the adaptability and generalization of DDP-based controllers.
Study~\cite{x6} presented a general parameterized version of differential dynamic
programming for solving problems with time-invariant parameters.
Study~\cite{ImDDP} ensured better numerical stability while handling dynamics and path constraints through primal-dual proximal Lagrangian Approach and advanced integrators.

\section{Optimal Control Problem in RETRO}
We consider the following general nonlinear discrete time control dynamical system
\begin{equation}
    \label{eq:system_dynamics}
    {x}_{t+1} = f({x}_t, {u}_t) +\omega_t,
\end{equation}
where ${x}_{t} \in \mathbb{R}^{n}$ and ${u}_{t} \in \mathbb{R}^{m}$ are the system state and the control input at time-step $t$, correspondingly. The nonlinear system transition function $f(\cdot,\cdot): \mathbb{R}^{n} \times \mathbb{R}^{m} \xrightarrow []{} \mathbb{R}^{n}$ is smooth and assumed to have second order derivatives. $\omega_t$ models independent and identically distributed Gaussian noise.  

Conventionally, the finite time optimal control problem giving~\eqref{eq:system_dynamics} is to find a control sequence ${U}^*=\{u^*_0,\cdots,u^*_T\}$ such that it minimizes the following cost function.
\begin{equation}
    \label{eq:loss}
J({u}_0,\cdots,{u}_T)=L({x}_T)+\sum_{t=0}^{T-1}L({x}_t, {u}_t),
\end{equation}
where $L({x}_t,{u}_t)$ and $L({x}_T)$ are the running cost and final cost, respectively. Given the principle of optimality \cite{principle_of_opt}, one can use a dynamic programming algorithm such as DDP to minimize Eq.~\eqref{eq:loss} by solving a single time-step control action backward in time~\cite{liao1991convergence},  
\begin{equation}
\label{eq: traditional value function}
\begin{split}
   V(x_t,t) = \underset{u_t}{min} \{\underbrace{{L}(x_t,u_t) + V(x_{t+1},t+1)}_{Q(x_t,u_t)}\}.
\end{split}
\end{equation} 
Here, we consider a quadratic Taylor expansion of a sufficiently small perturbation of $Q(x_t + \delta x_t,u_t + \delta u_t)$ around a nominal point $[x_t,u_t]$,
\begin{equation}
\label{Quadratic expansion of Q}
\begin{array}{lcl}
      Q(x_t + \delta x_t,u_t + \delta u_t) &\approx& Q(x_t,u_t) + (Q^{x}_{t})^{'}\delta x_t +(Q^{u}_{t})^{'}\delta u_t \\
      &+& \frac{1}{2} \begin{pmatrix}
         (\delta x_t)^{'}\\
         (\delta u_t)^{'}\\ 
     \end{pmatrix}
     \begin{pmatrix}
         Q^{xx}_{t} & Q^{xu}_{t} \\ 
         Q^{ux}_{t} & Q^{xx}_{t} 
     \end{pmatrix}
     \begin{pmatrix}
         \delta x_t\\
         \delta u_t
     \end{pmatrix}
\end{array}
\end{equation}
where the superscripts of $Q$ are denotes the derivatives that can be solved analytically~\cite{liao1991convergence}. Using the above approximation Eq.~\eqref{Quadratic expansion of Q}, one can minimize $\delta Q_t = Q(x_t + \delta x_t,u_t + \delta u_t) - Q(x_t,u_t)$ with respect to $\delta u_t$ and get the conventional DDP optimal control update backward in time,
\begin{equation}
\label{DDP optimal control law}
    \delta u^{\text{DDP}}_t = K_r + K_g\delta x_t,
\end{equation}
where the feedforward gain $K_r = -(Q^{uu}_{t})^{-1}Q^{u}_{t}$ and feedback gain $K_g = -(Q^{uu}_{t})^{-1}Q^{ux}_{t}$. Next, a line search-based forward process will be computed to update the nominal trajectory, 
\begin{equation}
    \begin{array}{lcl}
         \hat{x}_0&=& x_0,   \\
         \hat{u}_{t}&=& u_t + \epsilon K_r + K_g(\hat{x}_{t} - x_t),\\
         \hat{x}_{t+1}&=& f(\hat{x}_{t},\hat{u}_{t}),\\
    \end{array}
\end{equation}
where $0< \epsilon <1$ is a scalar. The above process repeats until the value function converges (see Fig.~\ref{fig:comparison} left panel).

On the contrary, optimal control handling in Reactive Trajectory Optimization (RETRO) is versatile and dynamic since the target state ${x}_{T}$ is non-stationary and probabilistic. Thus, the corresponding control sequence needs to be adjusted against the distribution shift according to the objective target state ${x}_{T}$. Here, we defined the time-varying objective goal location at each discrete time step $t$ as $o_t$, which can be modeled by the time-series forecasting techniques such as the Gaussian mixture model (GMM) \cite{GMM} or Bayesian time series analysis\cite{steel2010bayesian}. Therefore, the prior value function in Eq.~\eqref{eq: traditional value function} becomes the following equation in the objective belief space, $V(o_t,t) = \underset{u_t}{min}\{{L}(o_t,u_t) + V(o_{t+1},t+1)\}$. As new observations appear from tracking non-stationary target locations, the objective belief space at each time step evolves which in turn updates the prior belief about the objective trajectory. Thus, we introduce the following Kullback-Leibler(KL) divergence to quantify the distribution shift between posterior belief $P^\prime(o_t|obs)$ and prior belief $P(o_t)$,
\begin{equation}
    \label{eq: KL Diverngece for the first time definition}
    D_{KL}(P^\prime(o_t|obs)|| P(o_t) = E_{o_t\sim P^\prime(\cdot)}\text{log} \left [\frac{P^\prime(o_t|obs)}{P(o_t)} \right ].
\end{equation}
With~\eqref{eq: KL Diverngece for the first time definition}, we modify the prior value function by adding the KL divergence and reformulate the optimal control problem in RETRO,
\begin{equation}
\begin{split}
   \Tilde{V}(o_t,t) = \underset{u_t}{min} \{E_{o_t\sim P^\prime(\cdot)}[\Tilde{L}(o_t,u_t) + \Tilde{V}(o_{t+1},t+1)] + \\ D_{KL}(P^\prime(o_t|obs)|| P(o_t))\}.
\end{split}
\label{ideal+KL: optimality condition}
\end{equation}  

\begin{figure}[htbp]
\centering
    \includegraphics[width=1\linewidth]{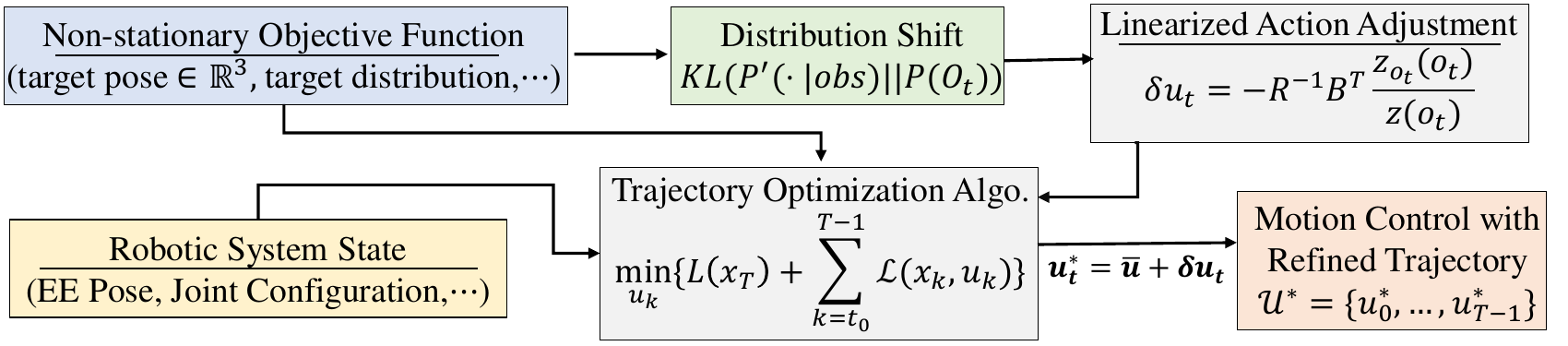}
    \caption{Overall Algorithmic Blocks of RETRO}
    \label{fig:algo}
\end{figure}
\color{black}
\section{Methodology of Reactive Trajectory Optimization (RETRO)}
Traditional DDP algorithms effectively solve the optimal control problem under the scenario where the tracking target is stationary and deterministic. However, when the evolution of the target trajectory is non-stationary and uncertain, it becomes challenging to run the DDP algorithm to achieve optimal control sequence as well as realize expected tracking performance. One naive solution to mitigate uncertainties is to repeat performing the DDP algorithm at certain time intervals as we have more information about future trajectories. However, running the DDP algorithm multiple times is very computation-heavy and there remains the risk of solving a online dynamic robotic task because of multiple forward-backward processes in control calculation\cite{Plancher18MastersThesis}. 

To address the above challenges, we develop an efficient control adjustment algorithm while maintaining the utilization of the DDP solutions. Specifically, when we receive a belief-space trajectory $\{o_0,\cdots, o_T\}$ for varying target objectives, the KL divergence enters into the value function Eq.~\eqref{ideal+KL: optimality condition} resulting in the prior DDP optimal solution becoming sub-optimal. Therefore, we perform control adjustments by exploiting linearly solvable optimal control formulation to fine-tune previously computed sub-optimal control sequences $\{\Bar{u}_0,\cdots,\Bar{u}_T\}$. Intuitively, our method utilizes the introduced KL divergence term to guide the refinement of control actions. Fig. \ref{fig:algo} depicts the overall methodology. We provide the detailed derivations of our algorithm in the following subsection. Although the derivations slightly follow the KL control theory \cite{Todorov2009, pmlr-v65-neu17a}, the control adjustment scenario is innovative in our work.

\subsection{Linearly Solvable Value Function Update}
We start by calculating the difference of the minimized value functions,
\begin{equation}
\begin{aligned}
     &\delta V(o_t,t) = \Tilde{V}(o_t,t) - V(o_t,t)  \\
     &= \delta L(t) + \delta V(o_{t+1},t+1) + D_{KL}(P^\prime(o_t|obs)|| P(o_t)]\\
     &= \delta L(t) + E_{o_t\sim P^\prime(\cdot)}\text{log} \left [\frac{P^\prime(o_t|obs)}{P(o_t)e^{-\delta V(o_{t+1},t+1)}} \right ],
\end{aligned}
\label{KL introduced: delta V}
\end{equation}
where $\delta V(o_{t+1},t+1)=E_{o_t\sim P^\prime(\cdot)}[\Tilde{V}(o_{t+1},t+1)]-V(o_{t+1},t+1)$ and $\delta L(t)=E_{o_t\sim P^\prime(\cdot)}[\Tilde L(o_t,u_t)]-L(o_t,u_t)$. Next, we define the desirability function, 
\begin{equation}
\label{Desirability function: z}
   z(o_{t+1}) = e^{-\delta V(o_{t+1},t+1)}.
\end{equation}
Intuitively, $z(o_t)$ contains information about how much control adjustment is required at time-step $t$. We next define a normalization term for the KL divergence as,
\begin{equation}
\label{Renormalization term: g[z](o_t)}
   g[z](o_t) = \sum_{t}P(o_t)z(o_{t+1}).
\end{equation}
As a result, the Eq.\eqref{KL introduced: delta V} can be simplified as,
\begin{equation}
\begin{aligned}
     &\delta V(o_t,t) = \delta L(t) + E_{o_t\sim P^\prime(\cdot)}\text{log} \left [\frac{P^\prime(o_t|obs)}{P(o_t)z(o_{t+1})\frac{g[z](o_t)}{g[z](o_t)}} \right] \\
     &= \delta L(t) + E_{o_t\sim P^\prime(\cdot)}\text{log} \left [\frac{1}{g[z](o_t)} \right] + D_{KL}(P^\prime(o_t|obs)|| \frac{P(o_t)z(o_{t+1})}{g[z](o_t)}).\\
\end{aligned}
\label{KL introduced simplified: delta V}
\end{equation}

\begin{figure}[htbp]
\centering
    \includegraphics[width=1\linewidth]{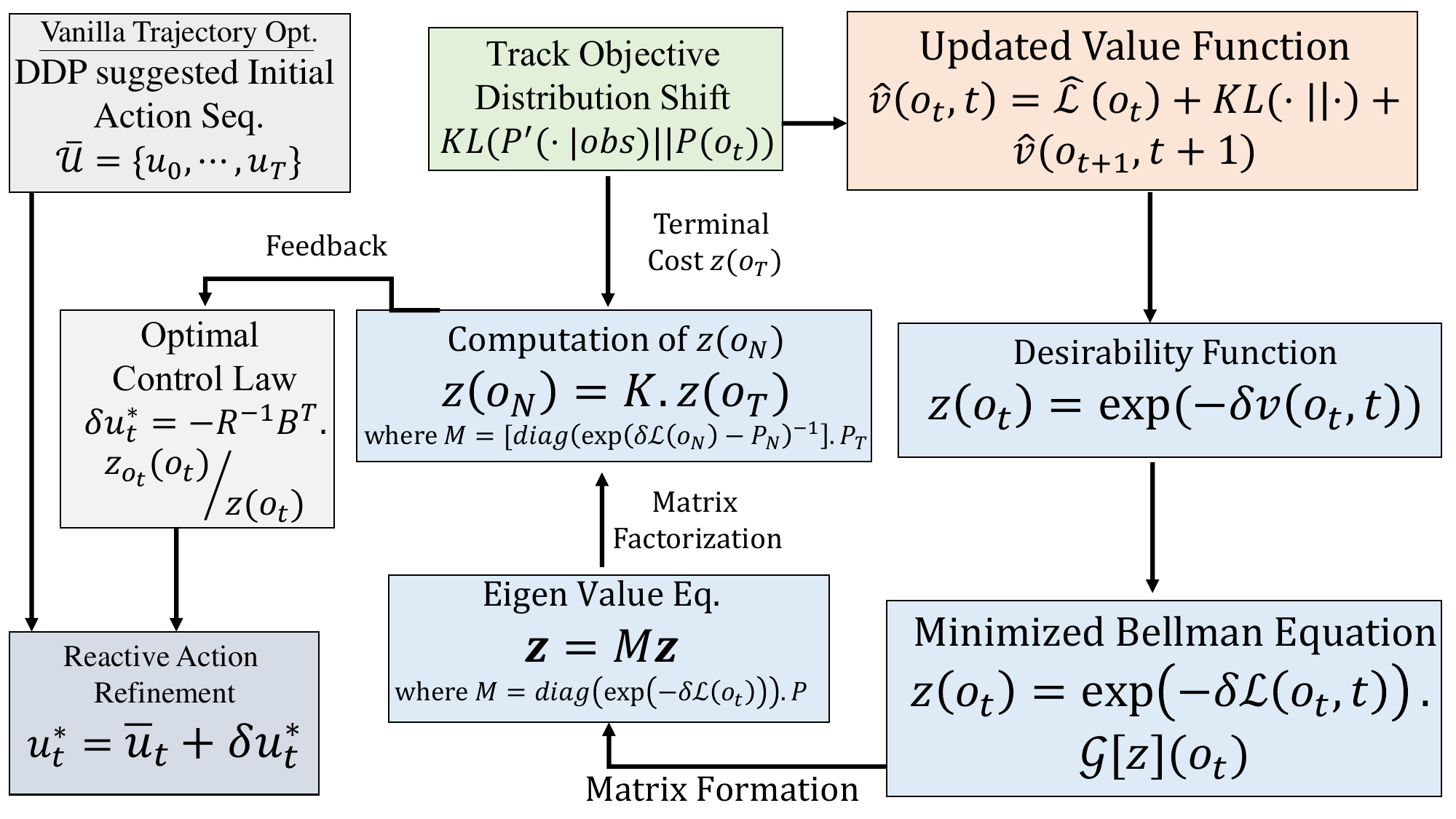}
    \caption{RETRO diagram for control adjustment}
    \label{fig:main}
\end{figure}

To minimize the additional KL divergence introduced term $D_{KL}(P^\prime(o_t|obs)|| \frac{P(o_t)z(o_{t+1})}{g[z](o_t)})$ in~\eqref{KL introduced simplified: delta V}, one can simplify and restrict the following termination conditions,
\begin{equation}
\label{Termination condition: z(t+1)}
z(o_{t+1}) = \frac{P^\prime(o_t|obs)}{P(o_t)}g[z](o_t),
\end{equation}
which results in the vanishing of KL divergence term since~\eqref{Termination condition: z(t+1)} ensures two identical distributions. Thus, we obtain the minimized value function update $\delta V(o_t,t)$ formulation as $\delta V(o_t,t) = \delta L(t) + \text{log}[\frac{1}{g[z](o_t)}]$. Using the definition of the desirability function~\eqref{Desirability function: z}, the minimized equation~\eqref{KL introduced simplified: delta V} can further be simplified under some algebraic operations (see Appendix for details),
\begin{equation}
\label{Linear solvable in z}
z(o_{t}) = e^{-\delta L(t)}g[z](o_t).
\end{equation}
The main advantage of the above relationship~\eqref{Linear solvable in z} is that the exponentiated minimized value function $z(o_{t})$ becomes linear in $z$ as the normalization term $g[z](o_t)$ can be treated as a linear operator. Essentially, it can be written in a more compact matrix form (see Appendix for details).
\begin{equation}
\label{Linear solvable in z: simplified}
z = Mz,
\end{equation}
where $M=diag(e^{-\delta L(1:T)}) P(O_{1:T})\in \mathbb{R}^{(T\times T)}$. Therefore, Eq.~\eqref{Linear solvable in z: simplified} is greatly simplified compared to~\eqref{Linear solvable in z} as in the vector form. To solve Eq.~\eqref{Linear solvable in z: simplified} efficiently, we define index sets $N$ and $T$ as a non-terminal trajectory index and terminal trajectory index, respectively. Thus, Eq.~\eqref{Linear solvable in z: simplified} can be partitioned as $\text{diag}(e^{\delta L(N)} - P(O_{N}))z(O_{N}) = P(O_{T})z(O_{T})$ via matrix partitioning , which gives the solution as 
\begin{equation}
\label{Linear solvable in z: simplified partitioned solution}
z(O_{N}) = (\text{diag}(e^{\delta L(N)}) - P(O_{N}))^{-1}P(O_{T})z(O_{T}).
\end{equation}
Once we calculate all $z$ values, i.e., $z(O_{1:T})$ for the pre-computed trajectory from DDP, we can associate the change of minimized value function $\delta V(o_t,t)$ with the desirability function $z(o_t,t)$ using the chain rule,
\begin{equation}
\label{delta V to delta z}
    \frac{\partial \delta V(o_t,t)}{\partial o_t} = \frac{\partial z(o_t,t)}{\partial o_t} \frac{1}{z(o_t)}.
\end{equation}
In the next section, We provide an example of fine-tuning the control sequence according to the adjustment of the value function.

\subsection{Adjustment of sub-optimal control sequence}
Considering the following control dynamics of our robotic arms with quadratic cost $L(x_t,u_t,t)= \frac{1}{2}u_{t}^{\top}R_t u_t + x_t^{\top}Q_t x_t $, the closed-form optimal control law solved from the Hamilton-Jacobi-Bellman(HJB) equation is given by,
\begin{equation}
\label{delta V to delta u}
    \begin{array}{lcl}
         dx &=& Axdt + Budt + \sigma d\omega,   \\

         u_{t}^{\text{optimal}} &=& -R_t^{-1}B^{\top}\frac{\partial V(x,t)}{\partial x}.
    \end{array}
\end{equation}
Together combining~\eqref{delta V to delta z} and~\eqref{delta V to delta u}  we calculate each step control adjustment as 
\begin{equation}  
\label{delta u}
    \delta u^{*}_{t} = -R_t^{-1}B^{\top}\frac{\partial \delta V(o_t,t)}{\partial o_t} = -R_t^{-1}B^{\top}\frac{\partial z(o_t,t)}{\partial o_t} \frac{1}{z(o_t)}.
\end{equation}
As a result, we can finetune the previous sub-optimal control law as  $u_{t}^{*} = \Bar{u}_{t}+\delta u_{t}^{*}$. To this end, we complete our RETRO algorithm design. A summary of the proposed algorithm can be found in the algorithm~\ref{Algorithm 1} and figure~\ref{fig:main}.
\begin{algorithm}
	\caption{RETRO Algorithm}
	\label{Algorithm 1}
		\textbf{Input:} 
  \\Initial Control Seq. $\Bar{U}=\{\Bar{u}_0,\cdots,\Bar{u}_T\}_{\text{DDP}}$  \\ Prior Belief Traj. $B=\{P(o_0),\cdots,P(o_T)\}$\\
  Shift Threshold, $\lambda_{thres}$\\
  \textbf{Output:}
  \\Optimal Control Sequence, $U^*$\\
  \SetKwProg{Fn}{Function}{}{}
\Fn {\text{Online trajectory optimization($\Bar{U}, B, \lambda_{thres}$):}}{
		\For{$t \gets 1$ to $T$} {
        \texttt{GET} : $\{obs\}$ of target function \\
        \texttt{UPDATE} : Calc. Posterior $P^\prime(o_t|obs)$ \\
        \texttt{CALCULATE} : $D_{KL}(P^\prime(o_t|obs)||P(o_t))$
        \\\uIf{$D_{KL}(\cdot||\cdot)>\lambda_{thres}$}{
            \texttt{CALCULATE} : $z(o_t)$, $\delta u_t^*$ \\
            \texttt{FINETUNE} : $u_{t}^{*}\gets\Bar{u}_{t}+\delta u_{t}^{*}$
        
        }
        \texttt{EXECUTE} : $u_{t}^{*}$
    }
} 
\end{algorithm}

\subsection{Regret Bound Analysis}
In this section, we conduct theoretical analyses to investigate the error performance of the proposed RETRO algorithm. We provide a regret-bound analysis compared with an Oracle value function solved by an Oracle trajectory optimization algorithm that completely knows the future evolution of the target function. First, we provide an intuition and observation. From Eq.~\eqref{KL introduced simplified: delta V}, we aim to cancel the additional KL divergence cost but end with another normalization term ${g[z](o_t)}$. It should be understood that this term will accumulate leading to $V > V^{*}$, where $V^{*}$ denotes the optimal value function achieved by the Oracle trajectory optimization algorithm. On the other hand, ${g[z](x_t)}$ should be bounded because of the following assumption and observation.
\begin{assumption}
\label{Assumption: Bounded KL Divergence}
Let $P(o_t)$ be the prior distribution of the object and $P^\prime(o_t|obs)$ be the updated posterior distribution based on new observations. We assume the distribution shift measured by $D_{KL}(P^\prime(o_t|obs)|| P(o_t))\ $ is smooth and the upper bound can be estimated by $\alpha(\frac{1}{T})$, 
\begin{equation}
    0\leq D_{KL}(P^\prime(o_t|obs)|| P(o_t))\ \leq \alpha(\frac{1}{T}).
\end{equation}
where $\alpha(.)$ is an increasing function. Essentially, the above assumption excludes the rapid change in the objective trajectory $o_t$ (see Appendix for details).
\end{assumption}
\begin{lemma}
\label{Lemma: Bounded Normalization Term}
The normalization term $g[z]o(t)$ defined in Eq.\eqref{Renormalization term: g[z](o_t)} by is bounded, i.e.,
\begin{equation}
    g[z](o_t) \leq Te^{-\delta V_m},
\end{equation}
where $T$ is the total time-step in the horizon of the control action sequence and $\delta V_m$ is the minimum value function deviated from the optimal value function $V^{*}$ across the total time-step.
\end{lemma}

\begin{proof}
From the definition of 
\begin{equation}
\begin{array}{lcl}
     g[z](o_t) &=& \sum_{t} P(o_t) z(o_{t+1}) \\
     &\leq& \sum_{t} z(o_{t+1})\; [ \text{since} \ (p(o_t))\leq 1] \\
     & = & e^{-\delta V_1} + e^{-\delta V_2} +  \cdots + e^{-\delta V_{T}}
     \;[\text{From def. of z}]\\
     &\leq & Te^{-\delta V_m} \;[\text{From def. of }\delta V_m]
\end{array}    
\end{equation}
where $m \in [1,T]$ such that $e^{-\delta V_m}$ achieves the maximum.
\end{proof}
Next, we describe the most essential theorem in stating the upper error bound. The proof of this theorem will use the above assumption~\ref{Assumption: Bounded KL Divergence} and lemma~\ref{Lemma: Bounded Normalization Term}.
\begin{theorem}
\label{Theorem: Regret is bounded}
The regret function $R_{t}$ defined as the expected deviation of the current value function $\Tilde{V}_t$ from the optimal value function $V^{*}_t$ is upper bounded, 
\begin{equation}
    \forall t\in [1,T] \quad R_{t} \leq \alpha(\frac{1}{T}) + \text{log}(T)
\end{equation}
\end{theorem}

\begin{proof}
We consider the proof by starting at time-step $m$, suggested by the previous lemma~\ref{Lemma: Bounded Normalization Term}. From the definition,
\begin{equation}
\begin{array}{lcl}
     R_{m} &=&  ||E_{o_m\sim P^\prime(\cdot)}[\Tilde{V}_m - V^{*}_m]||  = E_{o_m\sim P^\prime(\cdot)}||\delta V_m|| \\
     &=& E_{o_m\sim P^\prime(\cdot)}\text{log}(z(o_m)) \\
     &=& E_{o_m\sim P^\prime(\cdot)}\left[\text{log} \left [\frac{P^\prime(o_m|obs)}{P(o_m)}g[z](o_m) \right ]\right]  \\
     &\leq& D_{KL}(P^\prime(o_m|obs)|| P(o_m)) + E_{o_m\sim P^\prime(\cdot)}\text{log}(Te^{-\delta V_m})\\
     &\leq& \alpha(\frac{1}{T}) + \text{log}(T) - R_{m}\\
     &\leq& \frac{\alpha(\frac{1}{T})) + \text{log}(T)}{2}.
\end{array}    
\end{equation}
On the other hand, the lower bound of $R_{m}$ is 0 if there is no distribution shift, i.e., $D_{KL} = 0$. Together, we have the following
\begin{equation}
    0 \leq R_m  \leq \frac{\alpha(\frac{1}{T}) + \text{log}(T)}{2}.
\end{equation}
Thus, for a general time-step $t \in [1,T]$, $R_t \leq \alpha(\frac{1}{T}) + \text{log}(T) - R_m$, which gives us the final regret upper bound 
\begin{equation}
 \forall t\in [1,T] \quad R_t \leq \alpha(\frac{1}{T}) + \text{log}(T).
\end{equation}
\end{proof}
The above analysis provides an essential trade-off of our proposed algorithm. The regret error bound consists of two counterparts. The first term measures how large the distribution shift computed from the KL divergence. The larger $T$ we have, the smaller $\alpha(\frac{1}{T})$ is. On the contrary, the second part is related to the error bound across all time step $T$. The larger $T$ we have, the larger $\text{log}(T)$ has. Our empirical simulation evidence also indicates the above theoretical trade-off evidence (see Fig.~\ref{fig:Time}). In practice, researchers can decide the horizon $T$ depending on the problem.    

\subsection{Complexity Analysis}
We next analyze the computational complexity. The key ingredient of the RETRO algorithm is solving a linear equation~\eqref{Linear solvable in z: simplified} only once to get the fine-tuned control adjustment of all time steps, which clearly distinguishes the multiple-run DDP algorithm. We provide the following comparison between the multiple-run DDP and the RETRO algorithm.

The operation complexity of a single backward-forward DDP algorithm is given by $T(2n^{3} + \frac{7}{2}n^{3}m + 2nm^{3} + \frac{1}{3}m^{3})$~\cite{liao1991convergence}. In most cases, the dimension of the state variable $n$ is much higher than the dimension of the control variable $m$, i.e., $n\gg m$. Therefore, we can simplify the single backward-forward operation complexity as $O(Tn^3)$. Considering the quadratic convergence rate and termination bounds $O(\frac {1}{n^{2}T^{2}})$, we require at least $O(Tn)$ iterations before satisfying termination bounds (see Appendix for details). Together, the total iterations of the DDP algorithm can be estimated as $O(Tn\times Tn^{3}) = O(T^{2}n^{4})$.

On the other hand, our proposed method directly computes adjustment without any iterations. We solve Eq.~\eqref{Linear solvable in z: simplified} by inversing $T \times T$ matrix, which can be estimated as $O(T^{2.4})$. After that, the fine-tuned control adjustment is computed by Eq.~\eqref{Linear solvable in z: simplified partitioned solution} and Eq.~\eqref{delta u}, where the total operation complexity is by $O(Tn^{2})$. Together, the total running complexity is $O(T^{2.4}) + O(Tn^{2})$. Compared with the DDP algorithm $O(T^{2}n^{4})$, our methods slightly sacrifice $O(T^{0.4})$ for the first term because of the ambitious goal of finding fine-tuning control sequence without any iterations. However, our proposed method does not require computations $O(n^{4})$ on state dimensions for the first term and improves the second term by $O(Tn^{2})$ which together hugely improves the running time. The above comparison illustrates the superior complexity performance of our RETRO algorithm.

\section{EXPERIMENTAL PLATFORM}
We utilized the 7-DoFs Franka Emika Panda robotic manipulator as a simulation testbed to validate our algorithm. Each state sample ${x}_t\in\mathbb{R}^{13}$ consists of 13 numerical floating values, including $7$ joint angles, end-effector tip position coordinates in $\mathbb{R}^3$ and the most likely target coordinates in $\mathbb{R}^3$ estimated from GMM model~\cite{GMM}. We initialized the manipulator from a default \textit{home} position for all experiments, but the dynamic target trajectory was initiated from different locations with varying velocities for validating the generalizability of the proposed method. To meet the requirements of derivative computations by using an automatic differentiation tool and smooth calculation of system dynamics, we have modeled the entire simulation platform in the open-source robot simulator software Drake\cite{drake}. 

For real-hardware demonstrations, we used the same 7 DoFs Franka Emika robot manipulator fixed on a table-top as shown in Fig. \ref{fig:robot}. 
Our algorithm for object tracking and pose estimation, as well
as our adaptive trajectory optimization pipeline, ran on the QUAD GPU server equipped with an Intel Core-i9-9820X processor. We incorporated the Intel RealSense Depth Camera D435i for obstacle detection and localizing dynamic targets in $\mathbb{R}^3$. We tracked the target trajectory information through depth sensing in real-hardware experimentation and transformed the most probable intercepting point into the robot's coordinate system to catch the dynamic object smoothly in real time.

\section{RESULTS and DISCUSSIONS}
Through our experiments and simulation results, we wanted to investigate the following three questions, 
\begin{itemize}
    \item Can RETRO produce an optimized control sequence in real-time for intercepting the dynamic object?
    \item Can RETRO perform computation-efficient control adjustment in dynamic environments compared to the multiple-run iterative trajectory optimization process?
    \item Can RETRO exhibit trade-off performance as analysed by our theoretical proof and how does the control cost performance change related to time-horizon $T?$    
\end{itemize}
\begin{figure}[htbp]
\centering
    \includegraphics[width=1\linewidth]{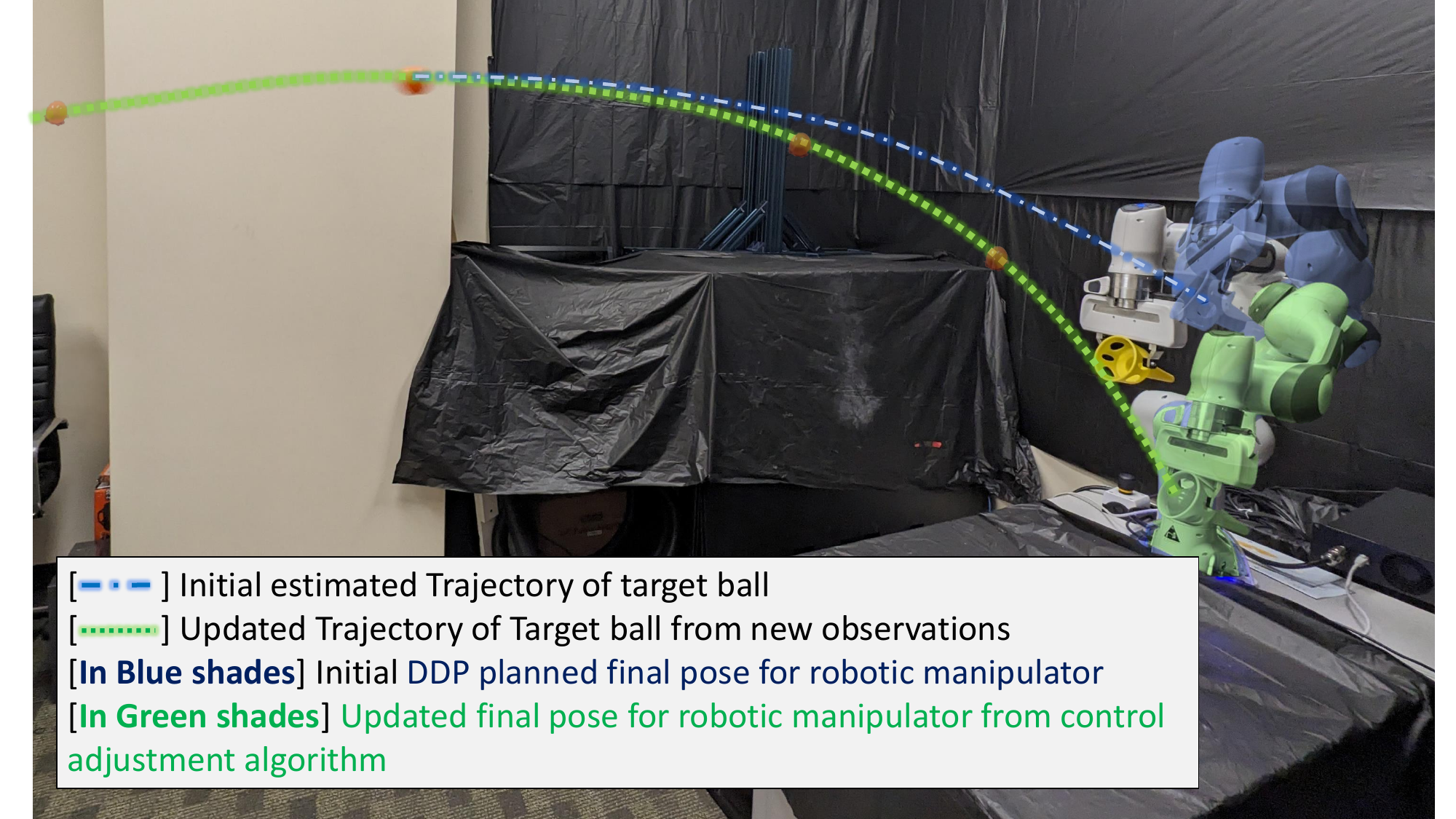}
    \caption{Real Hardware demonstration for robotic Manipulator for target catching task.}
    \label{fig:robot}
\end{figure}

\subsection{Successful Task Completion}
Our study aimed at leveraging the computation-efficient control adjustment technique for intercepting a flying target thrown toward the robotic manipulator. The object is thrown with different but bounded velocities from different initial locations. 
As discussed earlier, the object trajectory is tracked through a Realsense Depth sensor and the tracked trajectory points are given as time series inputs to our GMM model. The GMM model produces the probabilistic estimation for the most likely intercepting point. In Fig. \ref{fig:robot}, we have shown how in a real-hardware setting, the robotic manipulator followed the control sequence provided by RETRO. In the blue-shaded robot figure, we illustrate the sub-optimal final robot pose which is suggested by initial DDP-optimized control inputs. Since the probabilistic estimation gets updated with new target locations, the obsolete control inputs need to be upgraded. The in-queue control sequence gets refined in the direction given by the feedback accruing from KL divergence cost at the terminal state. In green shades, we show how the optimal final pose looks for the robotic manipulator and the robotic manipulator intercepts the target object(a low-weight ball here) with smooth trajectory execution.

\subsection{Efficient Control Computation}
\begin{figure}[htbp]
\centering
    \includegraphics[width=1\linewidth]{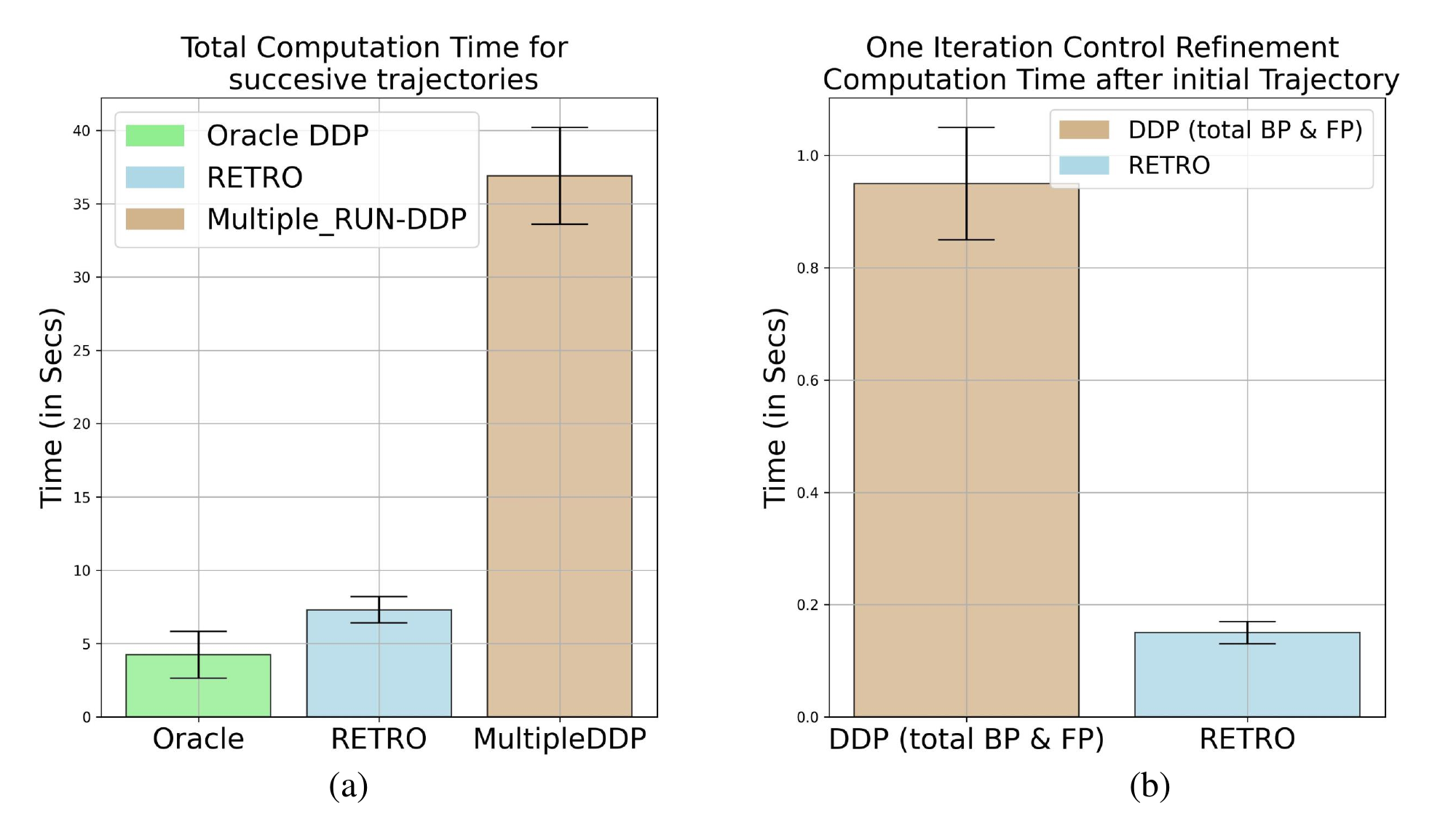}
    \caption{(a) Total Computation Time for Final Control Trajectories against an evolving target location. Total runtime is the highest for Multiple-Run DDP because for each distribution shift the controller need to perform multiple forward pass and backward pass before convergence. (b) Computation time for one single control adjustment iteration after tracking target distribution changes. Computation time for new control sequences is higher for DDP iteration as it computes new control sequences from scratch. In RETRO, we only adjust the pre-computed initial control sequence.}
    \label{fig:comp_time}
\end{figure}

Having established the successful task tracking performance, we next investigated how much computation efficiency the proposed method provides for control adjustment in a dynamic environment. To enable a fair comparison, we chose two baseline studies here: A) We chose an Oracle single-iteration DDP algorithm to which the full trajectory information and the accurate final landing point of the dynamic target are known beforehand. Simply, this Oracle DDP does not need to do any control input adjustment rather it gives an optimal control input sequence after meeting certain preset convergence criteria. B) We chose a computation-heavy multiple-run DDP optimization. The dynamic problem setting is as same as RETRO, but the method of re-optimization for handling non-stationary target locations is very different. For baselines, we completely discard the control sequence computed from a previous DDP run and start the optimization process from scratch for an updated target location. If the distribution shift is larger than a certain preset threshold, we run DDP again to find a new control sequence. In Fig. \ref{fig:comp_time}(a), we compare the computation time required for the mentioned baselines and the proposed method. As expected, the computation time is lowest for the Oracle method and highest for the multiple-run DDP method. Therefore, our method is computationally efficient with respect to the multiple-run DDP baseline since it performs very simple and efficient control input refinement for the already computed sub-optimal control sequence. On the other hand, the performance of our methods is close to the Oracle method which again confirms the superior performance. The above figure data are collected from Drake and ball movement simulations in a smoothly evolving trajectory. 

\subsection{Performance Evaluation for Varying Horizon}
Last, we investigated the effects of different preset horizons $T$ for total regret values and control costs between the RETRO and the Oracle algorithms. Our theoretical analyses assume the smoothness of the distribution shift of the computed KL divergence and envision the trade-off regret performance of control horizons. The KL divergence increases as we have a smaller horizon. Thus, the RETRO adjusted the control input irregularly and the prior sub-optimal control inputs were executed by the robotic system for larger intervals leading to the value of regret increased and larger control error cost differences (see Fig. \ref{fig:Time} left tail). On the contrary, as the horizon increases, the KL divergence becomes smoother leading to the difference in control cost shrinks. Therefore, the RETRO was only required to perform a very tiny control adjustment (see Fig. \ref{fig:Time} right tail). Further, We also observed that as the horizon increases, the total value of regret increases as consistent with our theoretical proof where the term $\text{log}T$ dominated the error upper bound. In between, We observed a relatively flat area where both the regret and control cost differences were small. In practice, how to select optimal times-step $T$ that balances the above trade-off will be a future research direction.   

\begin{figure}[htbp]
\centering
    \includegraphics[width=1\linewidth]{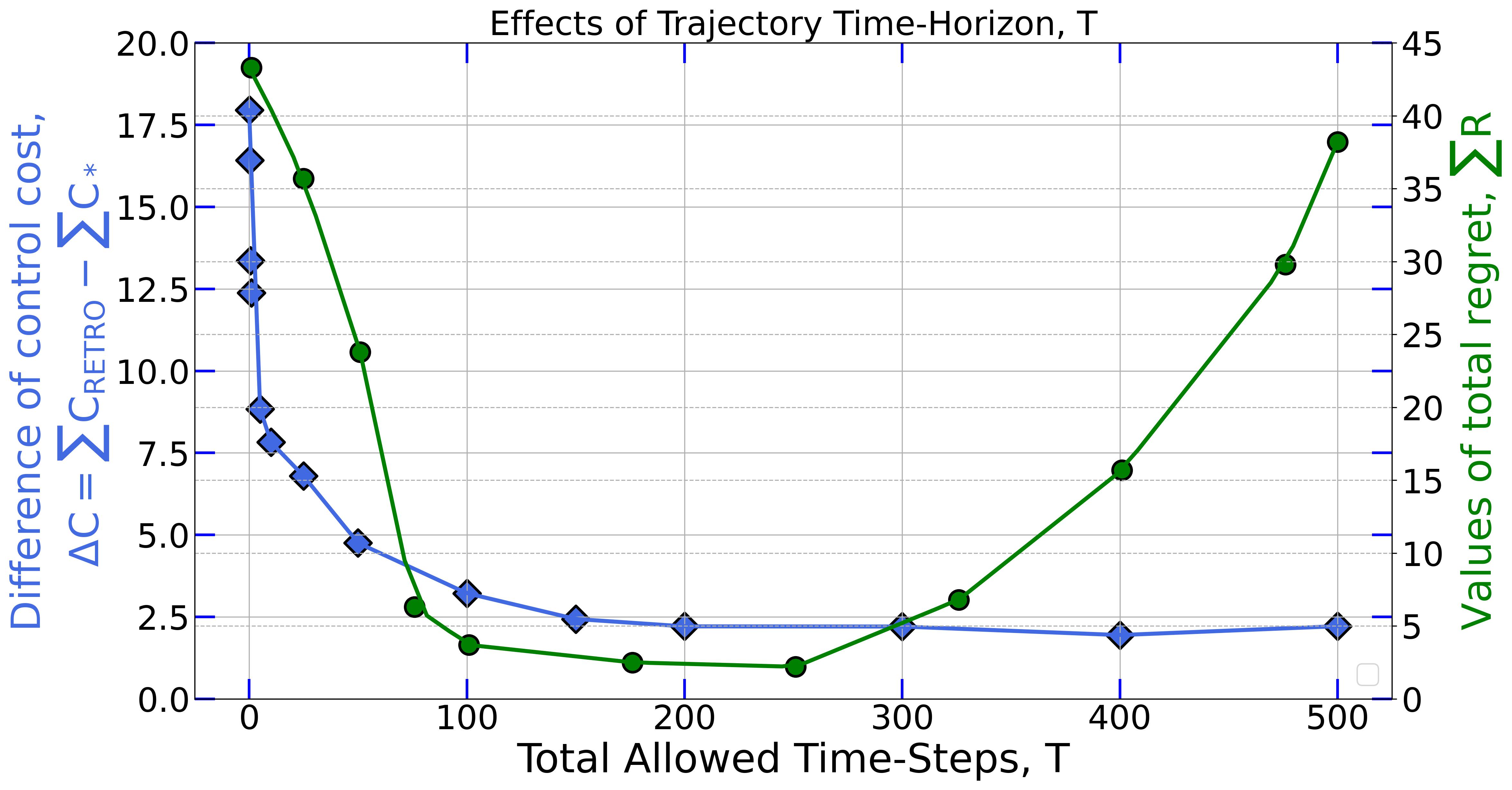}
    \caption{Relationship of planning horizon $T$ with the control cost difference and the total regret value. The control cost difference is in blue and the value of total regret is in green. Notice that the control cost difference becomes very tiny as $T$ increases but does not converge to $0$ as the linearized control refinement cost always exists.}
    \label{fig:Time}
\end{figure}

\section{CONCLUSIONS}
While much of the existing research has predominantly focused on optimal
control with predefined objectives, the emerging problems in robotic
trajectory optimization often involve objectives that evolve dynamically and
motion dynamics that exhibit stochastic behavior. Unfortunately, effectively
addressing such challenges for robot manipulators has proven difficult due to
issues related to high-dimensional trajectory representations and a lack of
consideration for time optimization. In response to these challenges, 
this paper has introduced a novel optimization framework, RETRO,  
to adaptively optimize
trajectories across both spatial and temporal dimensions, resulting in smooth
and optimally timed robot trajectories. Importantly, the RETRO
framework is task-agnostic, allowing it to seamlessly incorporate additional
task-specific requirements, such as collision avoidance, while maintaining
real-time control rates.

\bibliographystyle{IEEEtran}
\bibliography{Bibliography}

\section*{Appenxdices}

\setcounter{equation}{0}
\renewcommand{\theequation}{A.\arabic{equation}}
\paragraph*{\bf Appendix A}
\label{Appendix: Linear in Z}
{\bf Linear in Z}.
We provide detailed derivations for linear solvable value function updates, 
\begin{equation}
\begin{array}{lcl} 
 \delta V(o_t,t) &=& \delta L(t) + \text{log}[\frac{1}{g[z](o_t)}], \\
 -\delta V(o_t,t) &=& -\delta L(t) + \text{log}[g[z](o_t)], \\  
 e^{-\delta V(o_t,t)} &=& e^{-\delta L(t)} \times e^{\text{log}[g[z](o_t)]},\\
 z(o_{t}) &=& e^{-\delta L(t)}g[z](o_t).
\end{array}
\end{equation}
Consider the belief-space trajectories from time-step $1$ to $T$ and define the following matrix,
\begin{equation}
    \label{Transition matrix: P}
 P(O_{1:T}) =
 \begin{bmatrix}
  p(o_1) & p(o_2) & \cdots & \cdots & p(o_T) \\
  0 & p(o_2) & \cdots & \cdots &  p(o_T) \\
  \vdots& \ddots & \cdots & \cdots & \vdots \\
  \vdots& \vdots & \ddots & \cdots & \vdots \\
  \vdots& \vdots& \cdots& \ddots & \vdots \\
  0& 0&\cdots & \cdots& p(o_T)
 \end{bmatrix}
\end{equation} 
The $z(o_{t}) = e^{-\delta L(t)}g[z](o_t)$ can be written as 
\begin{equation}
\label{Linear solvable in z: simplified}
z = Mz,
\end{equation}
where $M=diag(e^{-\delta L(1:T)}) P(O_{1:T})\in \mathbb{R}^{(T\times T)}$.

\setcounter{equation}{0}
\renewcommand{\theequation}{B.\arabic{equation}}
\paragraph*{\bf Appendix B}
\label{Appendix: Upper Bound proof of KL Divergence}
{\bf Upper Bound proof of KL Divergence}.
\begin{theorem}
\label{Theorem: KL divergence of two Gaussian distribution}
Giving two Gaussian distributions $p(x)\sim \mathcal{N}(\mu_1,\,\sigma_{1})$and $q(x) \sim \mathcal{N}(\mu_2,\,\sigma_{2})$, the KL divergence is 
\begin{equation}
\label{KL divergence for two Gaussian}
 D_{KL}(p||q) = \text{log}(\frac{\sigma_2}{\sigma_1}) + \frac{{\sigma_1}^{2} + (\mu_1 - \mu_2)^{2}}{2{\sigma_2}^{2}} -\frac{1}{2}.
\end{equation}
\end{theorem}
Using the above theorem, we provide the following proof as well as provide the intuition of the upper bound $\alpha(\frac{1}{T})$. 

\begin{proof}
Without loss of generality, we assume both posterior and prior distribution to be Gaussian, i.e., $P^\prime(o_t|obs) \sim \mathcal{N}(\mu_1(T),\,\sigma_{1}(T))$and $P(o_t) \sim \mathcal{N}(0,1)$. Notice that the mean and variance of the posterior distribution should converge to the prior distribution as we increase the total time step $T$, i.e., $\lim_{T\to\infty} \mu_1(T) = 0$ and $\lim_{T\to\infty} \sigma_1(T) = 1$. Thus, we can simply model the above observation as $\mu_{1} = \frac{1}{T}$ and $\sigma_{1} = \frac{1}{T} + 1$. Therefore,        
\begin{equation}
\begin{array}{lcl}
     D_{KL}(P^\prime(o_t|obs)|| P(o_t)) &=& log(\frac{1}{\frac{1}{T} + 1}) +\frac{(\frac{1}{T} + 1)^{2}) + \frac{1}{T^{2}}-1}{2}\\
     &= &  \frac{1}{T} + \frac{1}{T^{2}} + log(\frac{T}{T+1}) \\
     &\leq & \frac{1}{T} + \frac{1}{T^{2}} \leq \alpha(\frac{1}{T}).
\end{array}    
\end{equation}
\end{proof}

\setcounter{equation}{0}
\renewcommand{\theequation}{C.\arabic{equation}}
\paragraph*{\bf Appendix C}
\label{Appendix: Expected iterations and operation complexity for running DDP before achieving termination bounds}
{\bf Expected iterations and operation complexity for running DDP before achieving termination bounds}\\
To evaluate the total expected iterations, we state the following theorem about the convergence evidence of the DDP algorithm~\cite{liao1991convergence}.
\begin{theorem}
There exists a constant $c>0$ such that the following inequality holds,
\begin{equation}
    \frac{||U^{J+1} - U^{*}||}{||U^{J} - U^{*}||^{2}} \leq c
\end{equation}
\end{theorem}
where $J$ denotes the $J^{th}$ iteration of the DDP algorithm and $U^{*}$ is the optimal control sequence. The above theorem indicates the quadratic convergence rate (similar to Newton's method~\cite{liao1991convergence}) for the DDP algorithm.
\begin{theorem}
Let $\phi\in {C}^{1}_{L}(\mathbb{R})$ (continuously differentiable function with Lipschitz continuity of gradient). Then the error in $\phi$ generated at each step $J$ by the optimization algorithm with quadratic convergence rate will be bounded by $O(\frac{1}{J^{2}})$.
\end{theorem}

Therefore, if we assume the terminated convergence error is bounded by $O(\frac{1}{n^{2}T^{2}})$ (we do not include the dimension of the control variable because we assume $n\gg m$ before), we can estimate the expected iterations of the DDP algorithm to be
$nT$. Notice that in many applications, the termination condition is even more restrictive than $O(\frac {1}{n^{2}T^{2}})$. For example, researchers usually set the $\text{error} \leq 10^{-8}$, which equals to run $10^4$ iterations to achieve the convergence for a control dynamical system with a state dimension of 10 and discrete time steps of 1000. Empirical experimental results also validated the estimation (see Results section). To this end, one can evaluate the running operation complexity for the multiple-run DDP algorithm as $O(Tn\times Tn^{3}) = O(T^{2}n^{4})$.

\end{document}